\documentclass{article}

\usepackage{arxiv}
\usepackage{graphicx}%
\usepackage{multirow}%
\usepackage{amsmath,amssymb,amsfonts}%
\usepackage{amsthm}%
\usepackage{mathrsfs}%
\usepackage[title]{appendix}%
\usepackage{xcolor}%
\usepackage{textcomp}%
\usepackage{manyfoot}%
\usepackage{booktabs}%
\usepackage{algorithm}%
\usepackage{algorithmicx}%
\usepackage{algpseudocode}%
\usepackage{listings}%
\usepackage{mathrsfs}
\usepackage[utf8]{inputenc} 
\usepackage[T1]{fontenc}    
\usepackage{hyperref}       
\usepackage{silence}
\usepackage{url}            
\usepackage{booktabs}       
\usepackage{amsfonts}       
\usepackage{nicefrac}       
\usepackage{microtype}      
\usepackage{lipsum}
\usepackage{graphicx}
\graphicspath{ {./images/} }

\usepackage{amsthm}

\theoremstyle{plain}
\newtheorem{theorem}{Theorem}[section]
\newtheorem{lemma}[theorem]{Lemma}

\newtheorem{corollary}[theorem]{Corollary}

\theoremstyle{definition}
\newtheorem{definition}[theorem]{Definition}

\theoremstyle{remark}
\newtheorem{remark}[theorem]{Remark}

\makeatletter
\renewcommand{\@toptitlebar}{}
\renewcommand{\@bottomtitlebar}{}
\makeatother

\usepackage{fancyhdr}
\title{Constructive Approximation of Random Process via Stochastic Interpolation Neural Network Operators}

\author{
  Sachin Saini \qquad Uaday Singh
  \\[1ex]
  Department of Mathematics\\
  Indian Institute of Technology Roorkee\\
  Roorkee, 247667, India
  \\[1ex]
  \href{mailto:sachin_saini@ma.iitr.ac.in}{\texttt{sachin\_saini@ma.iitr.ac.in}},
  \href{mailto:uaday.singh@ma.iitr.ac.in}{\texttt{uaday.singh@ma.iitr.ac.in}}
}

\begin{document}
\maketitle
\begin{abstract}
In this paper, we construct a class of stochastic interpolation neural network operators (SINNOs) with random coefficients activated by sigmoidal functions. We establish their boundedness, interpolation accuracy, and approximation capabilities in the mean square sense, in probability, as well as path-wise within the space of second-order stochastic (random) processes \( L^2(\Omega, \mathcal{F},\mathbb{P}) \). Additionally, we provide quantitative error estimates using the modulus of continuity of the processes. These results highlight the effectiveness of SINNOs for approximating stochastic processes with potential applications in COVID-19 case prediction.
\end{abstract}


\section{Introduction}
In recent years, artificial neural networks have successfully advanced in function approximation, pattern recognition, and machine learning.  
Research by Cybenko \cite{cybenko1989approximation} and Funahashi \cite{funahashi1989approximate} demonstrated that single-layer feed-forward neural networks (FNNs) can approximate continuous deterministic functions with arbitrary precision, provided they have a sufficient number of hidden neurons.  
These networks are now widely applied in computer science, biology, mathematics, physics, and engineering.  
Typically, an FNNs with one hidden layer can be mathematically expressed as
\begin{equation}\label{eq.1}
     \mathbf{N}_n(\mathbf{t})=\sum\limits_{k=0}^{n}{c_k}\,\eta\left(\mathbf{a}_k\cdot\mathbf{t}+b_k\right),
\end{equation}
where $\mathbf{t} = (t_1, t_2, \ldots, t_d)$, $\mathbf{a}_k \in \mathbb{R}^d$, $c_k, b_k \in \mathbb{R}$, and $\eta$ denotes the activation function.

The theory of \emph{neural network operators} (NNOs) provides a constructive framework for approximating functions using FNNs.  
Cardaliaguet and Euvrard \cite{cardaliaguet1992approximation} were the first to distinguish two types of NNOs: bell-shaped and squashing-type.  
These operators offer explicit formulations for approximating a function and its derivatives through a single-hidden-layer neural network using suitable activation functions.  
Later, Anastassiou \cite{anastassiou1997rate, anastassiou2000rate} investigated the convergence rates of such operators and further extended the analysis for various activation functions \cite{anastassiou2011multivariatesigmoidal, anastassiou2011sig-univariate, anastassiou2011hperbolioc-univariate, anastassiou2011multivariate-hyper, anastassiou2013frac-normal, anastassiou2013multi-vari-rate}.  

Deterministic interpolation-type NNOs activated by the ramp function were first defined by Costarelli \cite{costarelli2014interpolation-ramp, costarelli2015interpolation-ramp-multi} as
\begin{align}
    F_n(f,t)= \frac{\sum\limits_{k=0}^{n}f(t_k)\varphi_R\left(\frac{n(t-t_k)}{b-a}\right)}{\sum\limits_{j=0}^{n}\varphi_R\left(\frac{n(t-t_j)}{b-a}\right)}, \quad t\in [a,b],
\end{align}
where the nodes $t_k$ are uniformly spaced as $t_k=a+k\delta$ with $\delta=\tfrac{b-a}{n}$.  
Subsequently, Qian et al. \cite{qian2022rates-interpolation-A(m)-class} generalized this construction for a broader class of activation functions $\eta$, defining
\begin{align}
    S_n(f, t) = \sum_{k=0}^{n} f(t_k)\, \varphi_{\mathcal{A}(m)}\!\left(\frac{2m}{\delta}\left(t - t_k\right)\right), \quad t \in [a, b],
\end{align}
where $\varphi_{\mathcal{A}(m)}(t)=\eta(t+m)-\eta(t-m)$.  
Further developments include multivariate extensions \cite{wang2023Multi-INNOS-on-A(m)class,baxhaku2025multivariate,agrawal2025neural} and irregular grid formulations \cite{sharma2023approximation-irregular,costarelli2025higher,sharma2023fractional}and in fuzzy settings extension \cite{kadak2022multivariate}.   

\medskip
Despite their theoretical richness, most NNOs remain deterministic.  
However, many real-world signals and physical processes are inherently stochastic, with uncertainty arising from environmental noise or intrinsic randomness rather than measurement error.  
To address such cases, neural networks must be capable of learning mappings involving random functions and probabilistic dependencies.  
This perspective has motivated the development of \emph{stochastic neural networks} (SNNs), such as Boltzmann-type models introduced by Amari et al. \cite{amari1992information-boltezman} and recurrent stochastic networks explored by Zhao et al. \cite{zhao1996recurrent-SNN}.  
Belli et al. \cite{belli1999artificialNN} later established that SNNs can approximate stochastic processes in the mean square sense.  
Makovoz \cite{makovoz1996random-APP.with-NN} and Anastassiou \cite{anastassiou2022brownian-app.byNNO, anastassiou2024brownian-timeseprating, anastassiou2023approximation-TIMEsepatingNNOs-MDPI, anastassiou2023neural-Time-sepratingNNO} subsequently extended NNOs to random and time-separating processes.

Recent research trends also emphasize stochasticity and uncertainty quantification in neural network models.  
For instance, Gal and Ghahramani’s stochastic dropout framework has inspired further exploration of test-time uncertainty modeling through random weight injection \cite{ledda2023dropout}, while Durrmeyer-type deep neural networks \cite{kadak2025durrmeyer} established a rigorous connection between classical integral operators and neural architectures.  
Similarly, neural approaches have been utilized for modeling stochastic systems with jump dynamics \cite{ren2020modeling}, and fuzzy-stochastic hybrids have been developed for uncertainty-aware process modeling \cite{yuan2024interval}.  
From a theoretical standpoint, the expressive power of stochastic neural architectures has been analyzed using refined Kolmogorov complexity measures \cite{cabessa2025refined}, providing formal insight into their representational capabilities.  

\medskip
Motivated by these developments, in this paper we introduce a new class of \emph{stochastic interpolation neural network operators} (SINNOs)  
 which generalize traditional deterministic NNOs to the stochastic setting, providing an operator-based framework for approximating second-order processes \( X_t(\omega)\in L^2(\Omega,\mathcal{F},\mathbb{P}) \).  
The proposed SINNOs integrate stochasticity at the operator level through random coefficients \( X_{t_k}(\omega) \) while preserving the analytical and structural properties of deterministic interpolation operators.  
From an application standpoint, these operators effectively handle real-world signals, as demonstrated through their implementation on both simulated stochastic processes and real COVID-19 case data.

\section{Preliminaries}
Let $(\Omega, \mathcal{F}, \mathbb{P})$ be a probability space and let $\mathcal{T} = [0, T]$.
A \emph{stochastic process} is a family $(X_t)_{t \in \mathcal{T}}$ of random variables 
$X_t : (\Omega, \mathcal{F}) \to \mathbb{R}$. 
We say that $X_t$ is \emph{square-integrable} if $X_t \in L^2(\Omega, \mathcal{F}, \mathbb{P})$, i.e.,
\[
\mathbb{E}[|X_t|] < \infty, \qquad\mathbb{E}[|X_t|^2] < \infty.
\]
Throughout, we assume that for each $t \in \mathcal{T},$  $X_t \in L^2(\Omega, \mathcal{F}, \mathbb{P})$.
Here, $\Omega$ denotes the sample space, $\mathcal{F}$ is a $\sigma$-algebra of events, and $\mathbb{P}$ is a probability measure on $\mathcal{F}$.
The symbol $\mathbb{E}(\cdot)$ denotes the mathematical expectation.
See~\cite{doob1990stochastic} for foundational details on stochastic processes.

\begin{definition}
A measurable function $\eta : \mathbb{R} \to \mathbb{R}$ is called a \emph{sigmoidal function} if
\[
\lim_{t \to -\infty} \eta(t) = 0 
\quad \text{and} \quad 
\lim_{t \to +\infty} \eta(t) = 1.
\]
\end{definition}

\begin{definition}
Let $m > 0$ be fixed. 
A sigmoidal function $\eta$ is said to belong to the class $\mathcal{A}(m)$ if the following conditions hold:
\begin{enumerate}
    \item[$\mathcal{A}_1:$] $\eta(t)$ is non-decreasing;
    \item[$\mathcal{A}_2:$] $\eta(t)=1$ for $t \geq m$ and $\eta(t)=0$ for $t \leq -m$.
\end{enumerate}
\end{definition}

For $\eta \in \mathcal{A}(m)$, we define the corresponding \emph{activation function} by
\[
\varphi_{\mathcal{A}(m)}(t) = \eta(t+m) - \eta(t-m), \qquad t \in \mathbb{R}.
\]

\begin{remark}
Some examples of sigmoidal functions satisfying the above conditions are as follows:
\begin{enumerate}
    \item[(i)] \textbf{Ramp function:}
    \[
    \eta_R(t) =
    \begin{cases}
    0, & t \leq -\frac{1}{2},\\[4pt]
    t + \frac{1}{2}, & -\frac{1}{2} < t < \frac{1}{2},\\[4pt]
    1, & t \geq \frac{1}{2}.
    \end{cases}
    \]
    It is easy to verify that $\eta_R \in \mathcal{A}(1/2).$
    
    \item[(ii)] \textbf{Sigmoidal functions generated by central $B$-splines:}  
    The $B$-spline of order $r$ is defined by
    \[
    M_r(t) = \frac{1}{(r-1)!} 
    \sum_{j=0}^{r} (-1)^j \binom{r}{j} 
    \left(\frac{r}{2} + t - j\right)_{+}^{r-1},
    \]
    where $(t)_{+} := \max\{t, 0\}$ denotes the positive part of $t \in \mathbb{R}$.  
    The associated sigmoidal function is defined by
    \[
    \eta_{M(r)}(t) := \int_{-\infty}^{t} M_r(y) \, dy, \qquad t \in \mathbb{R}.
    \]
    Since $\operatorname{supp}(M_r) \subset [-r/2, r/2]$ and $M_r \ge 0$, 
    it follows that $\eta_{M(r)}(t) = 0$ for $t \le -r/2$ and 
    $\eta_{M(r)}(t) = 1$ for $t \ge r/2$. 
    Therefore, $\eta_{M(r)} \in \mathcal{A}(r/2).$
\end{enumerate}
\end{remark}

\begin{figure}[H]
    \centering
    \includegraphics[width=1\linewidth]{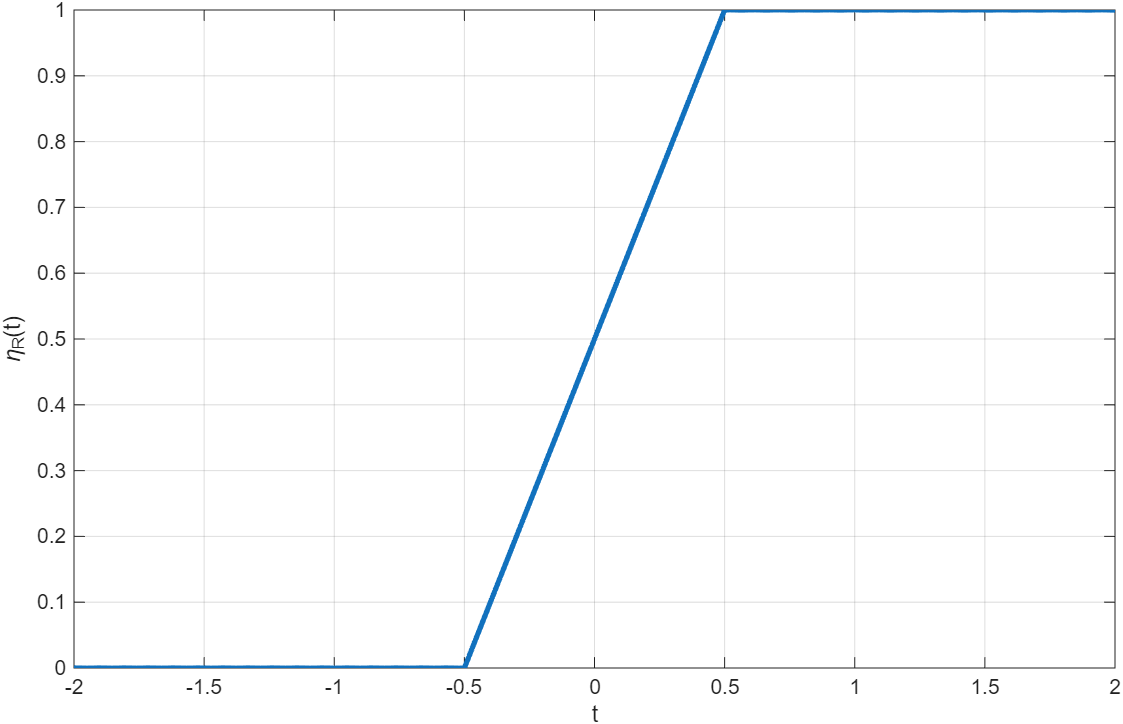}
    \caption{Sigmoidal ramp function $\eta_R$.}
    \label{ramp-plot}
\end{figure}

\begin{figure}[H]
    \centering
    \includegraphics[width=1\linewidth]{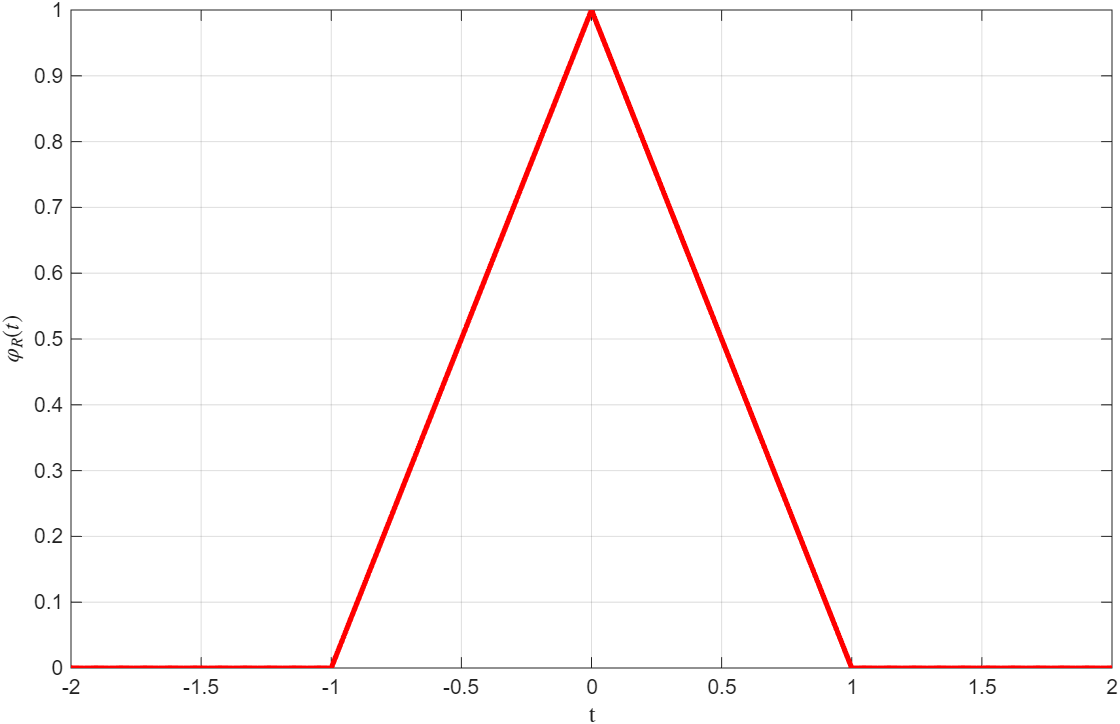}
    \caption{Activation function $\varphi_R$ corresponding to the ramp function $\eta_R$.}
    \label{Activation-plot}
\end{figure}

\noindent

\begin{definition}[Discrete absolute moments of order $\alpha$]
For $\alpha \ge 0$, the discrete absolute moments of order $\alpha$ for the activation function 
$\varphi_{\mathcal{A}(m)}$ are defined by
\[
\mathcal{M}_{\alpha}(\varphi_{\mathcal{A}(m)}) 
= \sup_{t \in \mathbb{R}} 
\sum_{k \in \mathbb{Z}} 
\big|\varphi_{\mathcal{A}(m)}(t - t_{k})\big| \, |t - t_{k}|^{\alpha}.
\]
\end{definition}

We summarize below some useful properties of $\varphi_{\mathcal{A}(m)}(t)$. 
See~\cite{qian2022rates-interpolation-A(m)-class, wang2023neural-interpolation-specific-sigmoidal} 
for detailed proofs.

\begin{lemma}\label{Lemma-properties-density}
If $\eta \in \mathcal{A}(m)$, then the activation function $\varphi_{\mathcal{A}(m)}$ satisfies:
\begin{enumerate}
    \item[\textnormal{(M1)}] $\varphi_{\mathcal{A}(m)}(t) \ge 0$ for all $t \in \mathbb{R}$;
    \item[\textnormal{(M2)}] $\varphi_{\mathcal{A}(m)}(t)$ is non-decreasing for $t < 0$ and non-increasing for $t \ge 0$;
    \item[\textnormal{(M3)}] $\operatorname{supp}(\varphi_{\mathcal{A}(m)}) \subseteq [-2m, 2m]$;
    \item[\textnormal{(M4)}] $\varphi_{\mathcal{A}(m)}(t) + \varphi_{\mathcal{A}(m)}(t - 2m) = 1$ for $t \in [0, 2m]$;
    \item[\textnormal{(M5)}] $\mathcal{M}_0(\varphi_{\mathcal{A}(m)}) < \infty$.
\end{enumerate}
\end{lemma}

\begin{lemma}[M5]\label{lem:discrete-moments-finite}
Let $\eta \in \mathcal{A}(m)$. 
For every $\alpha \ge 0$, the discrete absolute moment 
$\mathcal{M}_{\alpha}(\varphi_{\mathcal{A}(m)})$ is finite and satisfies
\[
\mathcal{M}_{\alpha}(\varphi_{\mathcal{A}(m)})
\le (2m)^{\alpha}\, \big(\lfloor 4m \rfloor + 2\big).
\]
In particular, $\mathcal{M}_{0}(\varphi_{\mathcal{A}(m)}) 
\le \lfloor 4m \rfloor + 2 < \infty$, which verifies \textnormal{(M5)}.
\end{lemma}

\begin{proof}
By (M3), $\operatorname{supp}(\varphi_{\mathcal{A}(m)}) \subset [-2m, 2m]$.
So, for any fixed $t \in \mathbb{R}$,
\[
\varphi_{\mathcal{A}(m)}(t - k) \neq 0 
\quad \Longrightarrow \quad |t - k| \le 2m.
\]
Thus, only integers $k$ in the interval $[t - 2m,\, t + 2m]$ contribute to the sum.
The number of such integers is at most $\lfloor 4m \rfloor + 2$.
Using $0 \le \varphi_{\mathcal{A}(m)}(t) \le 1$, we have
\begin{align*}
 \sum_{k \in \mathbb{Z}} 
|\varphi_{\mathcal{A}(m)}(t - k)| \, |t - k|^{\alpha}
&\le
\sum_{\substack{k \in \mathbb{Z} \\ |t - k| \le 2m}} |t - k|^{\alpha}&&\\ 
&\le (2m)^{\alpha} (\lfloor 4m \rfloor + 2).   
\end{align*}

Taking the supremum over $t \in \mathbb{R}$ yields the stated bound.
\end{proof}

\begin{definition}[Mean Square Convergence]\label{def:mean-square-convergence}
Let $\{X_n(t,\omega)\}_{n\in\mathbb{N}}$ be a sequence of stochastic processes and let $X(t,\omega)$ be another stochastic process defined on the same probability space $(\Omega, \mathcal{F}, \mathbb{P})$. 
We say that $X_n$ \emph{converges to} $X$ \emph{in the mean square sense} (or \emph{in $L^2$}) if, for every fixed $t$ in the domain of definition,
\[
\lim_{n \to \infty} 
\mathbb{E}\big[\,|X_n(t,\omega) - X(t,\omega)|^2\,\big] = 0.
\]
Equivalently, $X_n \to X$ in mean square if and only if
\[
\| X_n(t,\cdot) - X(t,\cdot) \|_{L^2(\Omega)} 
:= \big( \mathbb{E}[|X_n(t,\omega) - X(t,\omega)|^2] \big)^{1/2} \longrightarrow 0
\quad \text{as } n \to \infty.
\]
\end{definition}

\section{Stochastic Interpolation Neural Network Operators}
In this section, we introduce the SINNOs with random coefficients. We then establish their well-definedness and boundedness, followed by deriving their approximation properties in the mean-square sense, in probability and along individual sample paths.

\begin{definition}\label{SINNO-definition}
Let $(X_t)_{t \in \mathcal{T}}$ be a stochastic process defined on the probability space $(\Omega, \mathcal{F}, \mathbb{P})$, and let $\eta \in \mathcal{A}(m)$.
The corresponding \emph{stochastic interpolation neural network operators} (SINNOs) with random coefficients activated by  $\varphi_{\mathcal{A}(m)}$ are defined by
\begin{equation}\label{eq:SINNO}
\mathcal{S}_n(X_t, t)
:= \sum_{k=0}^{n} X_{t_k}(\omega)\,
\varphi_{\mathcal{A}(m)}\!\left(\frac{2m}{\delta}(t - t_k)\right),
\qquad t \in \mathcal{T},\; \omega \in \Omega,
\end{equation}
where the nodes $\{t_k\}_{k=0}^{n}$ are uniformly spaced points in $\mathcal{T} = [0,T]$ given by
\[
t_k = k\delta, \qquad k = 0, 1, \dots, n,
\]
and $\delta = \frac{T}{n}$ denotes the uniform step size.

The operator $\mathcal{S}_n(X_t, t)$ acts as a stochastic interpolant of the process $X_t(\omega)$, 
where each realization of the random coefficients $X_{t_k}(\omega)$ is modulated by the compactly supported activation function $\varphi_{\mathcal{A}(m)}$. The structure of the above-defined SINNOs is illustrated in Figure \ref{fig:diagramn_SINNOs.}.
\end{definition}

\begin{figure}[H]
    \centering
    \includegraphics[width=1\linewidth]{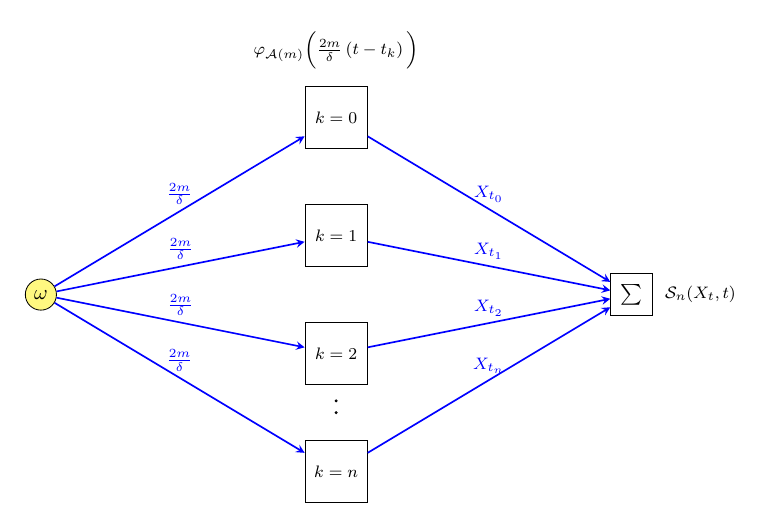}
    \caption{Structure of the above-defined SINNOs~$\mathcal{S}_n(X_t,t)$.}
    \label{fig:diagramn_SINNOs.}
\end{figure}
Firstly, we show that the operator $\mathcal{S}_n(X_t,t)$ is well-defined and bounded in the mean-square sense within the space $L^2(\Omega, \mathcal{F}, \mathbb{P})$.

\begin{theorem}[Mean-square boundedness of SINNOs]
\label{thm:bounded-SINNO}
Let $(X_t)_{t \in \mathcal{T}}$ be a stochastic process such that 
$X_t \in L^2(\Omega, \mathcal{F}, \mathbb{P})$ for each $t \in \mathcal{T}$, and let 
$\varphi_{\mathcal{A}(m)}$ be the activation function corresponding to a sigmoidal function 
$\eta \in \mathcal{A}(m)$.
Then, for every $t \in \mathcal{T}$,
\begin{equation}\label{eq:L2-boundedness}
    \mathbb{E}\!\left[ \big| \mathcal{S}_n(X_t, t) \big|^2 \right]
    \;\leq\;
    \mathcal{M}_0^2(\varphi_{\mathcal{A}(m)}) 
    \sup_{s \in \mathcal{T}} \mathbb{E}[|X_s|^2],
\end{equation}
where $\mathcal{M}_0(\varphi_{\mathcal{A}(m)})$ denotes the discrete absolute moment of order zero
(see Lemma~\ref{lem:discrete-moments-finite}).
\end{theorem}

\begin{proof}
From Definition~\ref{SINNO-definition}, we have
\[
\mathcal{S}_n(X_t, t)
= \sum_{k=0}^{n} X_{t_k}(\omega)
  \,\varphi_{\mathcal{A}(m)}\!\left(\tfrac{2m}{\delta}(t - t_k)\right).
\]
Since each $X_{t_k} \in L^2(\Omega, \mathcal{F}, \mathbb{P})$ and $\varphi_{\mathcal{A}(m)}$ is bounded,
$\mathcal{S}_n(X_t, t)$ is a finite linear combination of $L^2$ random variables
and hence belongs to $L^2(\Omega, \mathcal{F}, \mathbb{P})$.

Expanding the square and applying linearity of expectation gives
\begin{align*}
\mathbb{E}\!\left[\big| \mathcal{S}_n(X_t,t) \big|^2\right]
&= \mathbb{E}\!\left[
    \sum_{k=0}^{n}\sum_{j=0}^{n}
    X_{t_k}(\omega)\, X_{t_j}(\omega)\,
    \varphi_{\mathcal{A}(m)}\!\left(\tfrac{2m}{\delta}(t-t_k)\right)
    \varphi_{\mathcal{A}(m)}\!\left(\tfrac{2m}{\delta}(t-t_j)\right)
  \right] \\[3pt]
&= \sum_{k=0}^{n}\sum_{j=0}^{n}
   \mathbb{E}\!\left[X_{t_k} X_{t_j}\right]
   \varphi_{\mathcal{A}(m)}\!\left(\tfrac{2m}{\delta}(t-t_k)\right)
   \varphi_{\mathcal{A}(m)}\!\left(\tfrac{2m}{\delta}(t-t_j)\right).
\end{align*}

Using the Cauchy-Schwarz inequality,
\[
|\mathbb{E}[X_{t_k}X_{t_j}]|
   \le \big(\mathbb{E}[|X_{t_k}|^2]\big)^{1/2}
        \big(\mathbb{E}[|X_{t_j}|^2]\big)^{1/2}
   \le \sup_{s\in\mathcal{T}}\mathbb{E}[|X_s|^2].
\]
Hence,
\begin{align*}
\mathbb{E}\!\left[\big| \mathcal{S}_n(X_t,t) \big|^2\right]
&\le \sup_{s\in\mathcal{T}}\mathbb{E}[|X_s|^2]
     \left(
       \sum_{k=0}^{n}
       \big|\varphi_{\mathcal{A}(m)}\!\left(\tfrac{2m}{\delta}(t-t_k)\right)\big|
     \right)^{\!2}.
\end{align*}
By Lemma~\ref{lem:discrete-moments-finite} (for $\alpha=0$),
\[
\sup_{t\in\mathbb{R}}
\sum_{k\in\mathbb{Z}}
|\varphi_{\mathcal{A}(m)}(t-t_k)|
\le \mathcal{M}_0(\varphi_{\mathcal{A}(m)}),
\]
which yields the desired bound.
\end{proof}

We will now explore the interpolation and approximation properties of the sequence of operators, $\mathcal{S}_n(X_t,t)$ defined in (\ref{SINNO-definition}) in the mean square sense.

\begin{theorem}
\label{Theorem.L2-interpolation}
Let \((X_t)_{t\in\mathcal T}\) be such that \(X_t\in L^2(\Omega,\mathcal F,\mathbb P)\)
for every \(t\), and let \(\eta\in\mathcal A(m)\). Then for each node
\(t_i=i\delta\) \((i=0,1,\dots,n)\),
\[
\mathbb{E}\!\big[|\mathcal S_n(X_t,t_i)-X_{t_i}|^2\big]=0.
\]
\end{theorem}

\begin{proof}
Fix a node \(t_i\). Write the interpolation error at \(t_i,\) as a random
variable
\[
\mathcal E(\omega)
:= \mathcal S_n(X_t,t_i) - X_{t_i}(\omega)
= \sum_{k=0}^n \big(X_{t_k}(\omega)-X_{t_i}(\omega)\big)
  \varphi_{\mathcal A(m)}\!\Big(\tfrac{2m}{\delta}(t_i-t_k)\Big).
\]
Because each \(\varphi_{\mathcal A(m)}(\cdot)\) is deterministic, we expand
the mean square of \(\mathcal E\) and shift the deterministic terms outside
the expectation:
\begin{align}\nonumber
    \mathbb{E}\!\big[|\mathcal E|^2\big]
&= \mathbb{E}\bigg[\bigg(\sum_{k=0}^n (X_{t_k}-X_{t_i})\varphi_k\bigg)
           \bigg(\sum_{j=0}^n (X_{t_j}-X_{t_i})\varphi_j\bigg)\bigg]&& \\[3pt] \label{eq:double_sum}
&= \sum_{k=0}^n\sum_{j=0}^n \varphi_k\varphi_j\,
  \mathbb{E}\big[(X_{t_k}-X_{t_i})(X_{t_j}-X_{t_i})\big],
\end{align}
where we abbreviated
\(\varphi_k:=\varphi_{\mathcal A(m)}\!\big(\tfrac{2m}{\delta}(t_i-t_k)\big)\).
Now using the support property (M3), for any \(k\neq i,\) we have
\[
\Big|\tfrac{2m}{\delta}(t_i-t_k)\Big|=2m|i-k|\ge 2m,
\]
 
so \(\varphi_k=0\) for all \(k\neq i\).
Consequently every term in \eqref{eq:double_sum} with \(k\neq i\) or \(j\neq i\)
vanishes, leaving only the \((k,j)=(i,i)\) term:
\[
\mathbb{E}\!\big[|\mathcal E|^2\big] = \varphi_i^2\,
\mathbb{E}\big[(X_{t_i}-X_{t_i})^2\big] = \varphi_i^2\cdot 0 = 0.
\]
Therefore, \(\mathbb{E}[|\mathcal S_n(X_t,t_i)-X_{t_i}|^2]=0\), as required.
\end{proof}

\begin{theorem}
\label{thm:constant-reproduction}
If $X_t(\omega)=1$ for all $t\in\mathcal T$ and $\omega\in\Omega$, then
\[
\mathbb{E}\!\big[|\mathcal S_n(1,t)-1|^2\big]=0,\qquad\forall\,t\in\mathcal T.
\]
\end{theorem}

\begin{proof}
Fix \(t\in\mathcal T\). Since \(\varphi_{\mathcal A(m)}\) is deterministic
(i.e. independent of \(\omega\)), the SINNOs applied to the constant function
reduces to the deterministic finite sum
\[
\mathcal S_n(1,t)
= \sum_{k=0}^n \varphi_{\mathcal A(m)}\!\Big(\tfrac{2m}{\delta}(t-t_k)\Big).
\]
Hence, the mean-square error is just the square of a deterministic number:
\[
\mathbb{E}\!\big[|\mathcal S_n(1,t)-1|^2\big]
= \Big(\sum_{k=0}^n \varphi_{\mathcal A(m)}\!\Big(\tfrac{2m}{\delta}(t-t_k)\Big)-1\Big)^2.
\]
If \(t\in[t_i,t_{i+1}],\) then by (M3) only the terms with \(k=i\) and
\(k=i+1\) can be nonzero. By the partition-of-unity property (M4) which proves the theorem.
\end{proof}

We will now examine the quantitative estimates for the sequence of operators $\mathcal{S}_n(X_t,t)$ defined in (\ref{SINNO-definition}) in the space $L^2 (\Omega, \mathcal{F}, \mathbb{P})$ and measure the rate of approximation in terms of modulus of continuity defined below.

Let $ X_t(\omega)\in L^2 (\Omega, \mathcal{F}, \mathbb{P})$. For $h>0$, the function 
 \[
\mathcal{W}\left(X_t,h\right) = \max \left\{ \mathbb E|X_t(\omega) - X_s(\omega)|^2 : t, s \in \mathcal{T}, |s - t| \leq h \right\},
\]
 is known as the modulus of continuity of $X_t(\omega).$

The following theorem outlines several key properties of the modulus of continuity defined above.
\begin{theorem} Let \( X_t(\omega)\in  L^2 (\Omega, \mathcal{F}, \mathbb{P}).\) Then
\begin{enumerate} 
     \item [(i)] \(\mathcal{W}\left(X_t,h\right)\) is non-decreasing in $h;$
     \item[(ii)] If \(X_t(\omega)\) is mean square continuous, then \(\mathcal{W}\left(X_t,h\right)\) is  continuous in $h,$ and \(\mathcal{W}\left(X_t,h\right) \to 0\) as $h\to 0_{+};$ 
     \item[(iii)] If \( X_t(\omega) \)  is a stochastic process with \[ \mathbb E|X_t(\omega) - X_s(\omega)|^2 \leq C |t - s|^\alpha \] for some constant \( C > 0 \) and exponent \( 0 < \alpha \leq 1 \), then the modulus of continuity \( \mathcal{W}(X_t, h) \) is bounded and  \[
\mathcal{W}(X_t, h) \leq C h^\alpha;
\]
\item[(iv)] Let  \( X_t(\omega) \) and \( Y_t(\omega) \) be two independent stochastic processes. Then 
\[
\mathcal{W}(X_t + Y_t, h) \leq \mathcal{W}(X_t, h) + \mathcal{W}(Y_t, h);
\]   
\item[(v)] Let \( X_t(\omega) \) and \( Y_t(\omega) \) be two dependent stochastic processes. Then 
\[
\mathcal{W}(X_t + Y_t, h) \leq \mathcal{W}(X_t, h) + \mathcal{W}(Y_t, h) + 2\sqrt{\mathcal{W}(X_t, h) \mathcal{W}(Y_t, h)}.
\]
\end{enumerate}
\end{theorem}

\begin{proof}
        \begin{enumerate}
            \item [(i)] It is obvious that, by the definition of \(\mathcal{W}\left(X_t,h\right)\), it takes the supremum over all pairs \(t,s\) such that \(|t-s| \leq h\). Therefore, increasing \(h\) only adds more terms to the maximization set, ensuring that \(\mathcal{W}\left(X_t,h\right)\) does not decrease.
            
\item[(ii)]By the definition of \( \mathcal{W}(X_t, h) \), we have  
\[
\mathcal{W}(X_t, h) = \max \left\{\mathbb E |X_t(\omega) - X_s(\omega)|^2 : t, s \in \mathcal{T}, |s - t| \leq h \right\}.
\]  
Since \( X_t(\omega) \) is mean-square continuous, it follows that for any \( \epsilon > 0 \), there exists some \( h_0 > 0 \) such that  
\[
\mathbb E |X_t(\omega) - X_s(\omega)|^2 < \epsilon \quad \text{whenever } |t - s| < h_0.
\]  
Taking the supremum over all pairs \( (t, s) \) such that \( |s - t| \leq h \), we obtain  
\[
\mathcal{W}(X_t, h) \to 0 \quad \text{as } h \to 0.
\]  
Thus, the modulus of continuity \( \mathcal{W}(X_t, h) \) is right-continuous at \( h = 0 \).

\item [(iii)]   
By the definition of the modulus of continuity, we have 
\[
\mathcal{W}(X_t, h) = \max \left\{ \mathbb E|X_t(\omega) - X_s(\omega)|^2 : |s - t| \leq h \right\}.
\]  
Given that \( \mathbb E|X_t(\omega) - X_s(\omega)|^2 \leq C |t - s|^\alpha, \) it follows that  
\[
\mathbb E |X_t(\omega) - X_s(\omega)|^2 \leq C h^\alpha \quad \text{for all } t, s \in \mathcal{T}, |s - t| \leq h.
\]  
Taking the supremum over all such pairs \( (t, s) \), we get,  
\[
\mathcal{W}(X_t, h) \leq C h^\alpha.
\]
\item[(iv)]
For two independent processes $X_t(\omega)$ and $Y_t(\omega)$, we have
\[
\mathbb E|X_t(\omega) + Y_t(\omega) - (X_s(\omega) + Y_s(\omega))|^2 = \mathbb E|X_t(\omega) - X_s(\omega)|^2 + E|Y_t(\omega) - Y_s(\omega)|^2.
\]
Thus, the maximum of the sum on the RHS over \( |t - s| \leq h \) will satisfy
\[
\mathcal{W}(X_t + Y_t, h) = \max \left\{ \mathbb E|X_t(\omega) - X_s(\omega)|^2 + \mathbb E|Y_t(\omega) - Y_s(\omega)|^2 : |t - s| \leq h \right\}.
\]
This implies
\[
\mathcal{W}(X_t + Y_t, h) \leq \mathcal{W}(X_t, h) + \mathcal{W}(Y_t, h).
\]

\item[(v)] 
For dependent processes, we use the triangle inequality
\begin{align}\nonumber
\mathbb E|X_t(\omega) + Y_t(\omega)&-(X_s(\omega) + Y_s(\omega))|^2 \leq \mathbb E|X_t(\omega) - X_s(\omega)|^2 & \\[2pt] \nonumber
& \qquad\qquad\qquad\qquad\qquad\qquad+ \mathbb E|Y_t(\omega) - Y_s(\omega)|^2 && \\[3pt] \nonumber
&\qquad\qquad+2 \sqrt{\mathbb E|X_t(\omega) - X_s(\omega)|^2 \mathbb E|Y_t(\omega) - Y_s(\omega)|^2}.&& \\ \nonumber
\end{align}
Taking the supremum over \( |t - s| \leq h \), we get
\[
\mathcal{W}(X_t + Y_t, h) \leq \mathcal{W}(X_t, h) + \mathcal{W}(Y_t, h) + 2\sqrt{\mathcal{W}(X_t, h) \mathcal{W}(Y_t, h)}.
\]
\end{enumerate}
\end{proof}

\begin{theorem}
\label{Theorem-L^2-quantitative}
Let \(X_t(\omega)\in L^2(\Omega,\mathcal F,\mathbb P)\) for every \(t\in\mathcal T\),
and let \(\eta\in\mathcal A(m)\). For \(\delta=T/n\) the SINNOs
\(\mathcal S_n(X_t,t)\) defined in \eqref{SINNO-definition} satisfies 
\[
\mathbb{E}\!\big[\,|\mathcal S_n(X_t,t)-X_t|^2\,\big]\;\le\;
\mathcal W(X_t,\delta), \mbox{ for every }
t\in\mathcal T.
\]
Consequently, if \(X_t\) is mean-square continuous, then for each fixed \(t,\)
\(\mathcal S_n(X_t,t)\xrightarrow{L^2} X_t\) as \(n\to\infty\). Moreover,
the convergence is uniform in \(t\); i.e.,
\[
\sup_{t\in\mathcal T}\mathbb{E}\!\big[\,|\mathcal S_n(X_t,t)-X_t|^2\,\big]
\;\le\; \mathcal W(X_t,\delta)\xrightarrow{}0, \;\mbox{ as } n\to\infty.
\]
\end{theorem}

\begin{proof}
Fix $t\in\mathcal T$ and let $t\in[t_i,t_{i+1}]$ for some $i\in\{0,\dots,n-1\}$. 
By property (M3) of Lemma~\ref{Lemma-properties-density} only the two nearest nodes contribute to the sum, so

\[
\mathcal S_n(X_t,t)
= X_{t_i}\,\varphi\!\Big(\tfrac{2m}{\delta}(t-t_i)\Big)
  + X_{t_{i+1}}\,\varphi\!\Big(\tfrac{2m}{\delta}(t-t_{i+1})\Big).
\]
 Using $(M4)$, we have the partition of unity as
\begin{align}\label{eq:partition of unity}
    \varphi_{\mathcal{A}(m)}\left(\frac{2m}{\delta}\left(t-t_i\right)\right)+\varphi_{\mathcal{A}(m)}\left(\frac{2m}{\delta}\left(t-t_{i+1}\right)\right)=1.
 \end{align}
Thus, we rewrite the approximation error as under.
\begin{align}\nonumber
    \mathbb E\Bigl|\mathcal{S}_n(X_t,t)-X_t(\omega)\Bigl|^2&=\mathbb E\left|{\sum\limits_{k=0}^{ n }X_{t_k}(\omega)\varphi_{\mathcal{A}(m)}\left(\frac{2m}{\delta}\left(t-t_k\right)\right)}-X_t(\omega)\right|^2&&\\[2pt] \nonumber
    &=\mathbb E\Bigl|X_{t_i}(\omega)\varphi_{\mathcal{A}(m)}\left(\frac{2m}{\delta}\left(t-t_i\right)\right)&&\\[2pt] \nonumber
    &\quad\qquad\qquad +X_{t_{i+1}}(\omega)\varphi_{\mathcal{A}(m)}\left(\frac{2m}{\delta}\left(t-t_{i+1}\right)\right)-X_t(\omega)\Bigl|^2&&\\[2pt] \nonumber
    &=\mathbb E\Bigl|\left(X_{t_i}(\omega)-X_t(\omega)\right)\varphi_{\mathcal{A}(m)}\left(\frac{2m}{\delta}\left(t-t_i\right)\right)&&\\[2pt] \nonumber
    &\quad\quad\qquad+\left(X_{t_{i+1}}(\omega)-X_t(\omega)\right)\varphi_{\mathcal{A}(m)}\left(\frac{2m}{\delta}\left(t-t_{i+1}\right)\right)\Bigl|^2&&\\[2pt] \nonumber
   &\leq \varphi_{\mathcal{A}(m)}\left(\frac{2m}{\delta}\left(t-t_i\right)\right)\mathbb E\left|X_{t_i}(\omega)-X_t(\omega)\right|^2 && \\[2pt] \nonumber
    &  \quad\qquad\qquad+\varphi_{\mathcal{A}(m)}\left(\frac{2m}{\delta}\left(t-t_{i+1}\right)\right)\mathbb E\left|X_{t_{i+1}}(\omega)-X_t(\omega)\right|^2&&\\ \nonumber
    &\leq \varphi_{\mathcal{A}(m)}\left(\frac{2m}{\delta}\left(t-t_i\right)\right) \mathcal{W}(X_t,\delta) && \\[2pt] \nonumber
    &\qquad\qquad\qquad\quad + \varphi_{\mathcal{A}(m)}\left(\frac{2m}{\delta}\left(t-t_{i+1}\right)\right) \mathcal{W}(X_t,\delta)&&\\[2pt] \nonumber
    & \leq \mathcal{W}(X_t, \delta),
\end{align}
in view of the equation $(\ref{eq:partition of unity})$ and convexity of expectation.

If \(X_t\) is mean-square continuous, \(\mathcal W(X_t,\delta)\to0\) as \(\delta=T/n\to0\) as \(n\to\infty\), which implies
\(\mathbb{E}[|\mathcal S_n(X_t,t)-X_t|^2]\to0\) for each fixed \(t\). Taking
supremum in \(t\) yields the uniform statement
\(\sup_{t\in\mathcal T}\mathbb{E}[|\mathcal S_n(X_t,t)-X_t|^2]\le\mathcal W(X_t,\delta)\to0\),
so the convergence is uniform in \(t\) as claimed.
\end{proof}

\begin{corollary}[Rate of convergence under H\"older condition]
\label{cor:holder-rate}
Suppose there exist constants \(C>0\) and \(0<\alpha\le 1\) such that  
\[
\mathbb{E}\!\left[\,|X_t - X_s|^2\,\right] \le C |t-s|^\alpha,
\qquad \forall\, s,t \in \mathcal T.
\]
Then for every \(t \in \mathcal T\),
\[
\mathbb{E}\!\big[\,|\mathcal S_n(X_t,t) - X_t|^2\,\big]
\;\le\; C \delta^\alpha
= C\left(\frac{T}{n}\right)^\alpha.
\]
Consequently, \(\mathcal S_n(X_t,t) \to X_t\) in mean square with rate \(O(n^{-\alpha})\).
\end{corollary}

\begin{proof}
By definition of the modulus of continuity,
\[
\mathcal W(X_t,\delta)
= \sup_{|s-t|\le\delta}
   \mathbb{E}\!\big[\,|X_t - X_s|^2\,\big].
\]
Under the assumed H\"older condition,
\[
\mathbb{E}\!\big[\,|X_t - X_s|^2\,\big]
\;\le\; C |t-s|^\alpha
\;\le\; C\delta^\alpha,
\qquad \text{whenever } |t-s|\le\delta.
\]
Taking supremum over all such \(s\) yields
\[
\mathcal W(X_t,\delta)\le C\delta^\alpha.
\]

Applying Theorem~\ref{Theorem-L^2-quantitative}, we obtain
\[
\mathbb{E}\!\big[\,|\mathcal S_n(X_t,t)-X_t|^2\,\big]
\;\le\; \mathcal W(X_t,\delta)
\;\le\; C\delta^\alpha.
\]
Since \(\delta = \tfrac{T}{n}\), this becomes
\[
\mathbb{E}\!\big[\,|\mathcal S_n(X_t,t)-X_t|^2\,\big]
\;\le\; C\left(\frac{T}{n}\right)^\alpha.
\]

Thus,
\[
\mathbb{E}\!\big[\,|\mathcal S_n(X_t,t)-X_t|^2\,\big]
= O\!\left(n^{-\alpha}\right),
\]
which completes the proof.
\end{proof}

Now, we will explore the interpolating and convergence properties of our SINNOs defined in (\ref{SINNO-definition}) in probability.

\begin{theorem}[Uniform boundedness in probability]
\label{thm:bounded-in-prob}
Assume $(X_t)_{t\in\mathcal T}$ satisfies $\sup_{s\in\mathcal T}\mathbb{E}[|X_s|^2]<\infty$.
Then the sequence $(\mathcal S_n(X_t,t))_{n\ge1}$ is uniformly bounded in probability, i.e.,
for every $\varepsilon>0$ there exists $M>0$ such that
\[
\sup_{n\ge1}\mathbb{P}\big(|\mathcal S_n(X_t,t)|\ge M\big)\le\varepsilon.
\]
\end{theorem}

\begin{proof}
By Theorem~\ref{thm:bounded-SINNO} there exists a constant $K>0$
(independent of $n$) such that for all $n$ and every fixed $t$, we have
\[
\mathbb{E}[|\mathcal S_n(X_t,t)|^2]\le K,
\]
where $K=\mathcal M_0^2(\varphi_{\mathcal A(m)})\sup_{s\in\mathcal T}\mathbb{E}[|X_s|^2]$.
Using Markov (Chebyshev) inequality, we have
\[
\mathbb{P}\big(|\mathcal S_n(X_t,t)|\ge M\big)\le \frac{\mathbb{E}[|\mathcal S_n(X_t,t)|^2]}{M^2}\le \frac{K}{M^2}.
\]
Choosing \(M\) large enough ensures that the right-hand side is at most \(\epsilon\), proving uniform boundedness in probability.
\end{proof}

\begin{theorem}[Probability Convergence]\label{The.Probability Convergence3.6}
     If $X_t(\omega) \in L^2 (\Omega, \mathcal{F}, \mathbb{P}).$ Then for every $i=0,1,...,n,$ $\mathcal{S}_n(X_t, t_i)$ converges to $X_{t_i}(\omega)$ in probability as $n \to \infty$.
\end{theorem}

\begin{proof}
    For an arbitraty $\epsilon>0,$ using the Chebyshev's inequality, we have
\begin{align}\nonumber
\mathbb{P}\left(\left|\mathcal{S}_n(X_t,t_i)-X_{t_i}(\omega)\right|\geq \epsilon\right)\leq \frac{\mathbb E\left|\mathcal{S}_n(X_t,t_i)-X_{t_i}(\omega)\right|^2}{\epsilon^2}.
\end{align}    
In view of Theorem \ref{Theorem.L2-interpolation}, we conclude that interpolation in probability holds good.
\end{proof}

\begin{theorem}[Strong Convergence] If $X_t(\omega) \in L^2 (\Omega, \mathcal{F}, \mathbb{P}).$  Then $\mathcal{S}_n(X_t, t)$ converges to $X_t(\omega)$ almost surely.
\end{theorem}

\begin{proof}
    Define the event 
    $$ A_n= \left\{ \omega\in \Omega: \left|\mathcal{S}_n(X_t,t)-X_t(\omega)\right|\geq \frac{1}{n}\right\}.$$
 From probability convergence Theorem~\ref{The.Probability Convergence3.6}, we have
 $$\sum_{n=1}^{\infty}\mathbb{P}\left(A_n\right)<\infty.$$
 Applying the Borel-Cantelli Lemma, we conclude that
 $$\mathbb{P}(A_n \mbox{ occurs infinitely often })=0.$$
 Thus, for almost every $\omega,$ there exists $\lambda$ such that for all $n\geq \lambda,$
 $$\left|\mathcal{S}_n(X_t,t)-X_t(\omega)\right|<\frac{1}{n}.$$
 This completes the proof as $n\to \infty.$
\end{proof}

Now, we investigate the path-wise approximation capabilities of our stochastic interpolation-NNOs (SINNOs). 

For any fixed $\omega_i \in \Omega$, there exists a corresponding sample path $X_t(\omega_i)$ of the process $X_t(\omega)$. Depending on the properties of $X_t(\omega_i)$, we consider the following cases:  

\begin{enumerate}  
    \item[V1.] If $X_t(\omega_i) \in C(\mathcal{T})$, i.e., the sample path is continuous.  
    \item[V2.] If $X_t(\omega_i) \in AC(\mathcal{T})$, i.e., the sample path is absolutely continuous.  
    \item[V3.] If $X_t(\omega_i) \in L^p(\mathcal{T})$, i.e., the sample path belongs to the $L^p$ space.  
\end{enumerate}  

Under these conditions, the SINNOs introduced in (\ref{SINNO-definition}) reduce to deterministic interpolation-NNOs : 
\begin{equation}  
\mathcal{S}_n(X_t(\omega_i), t) = \sum\limits_{k=0}^{n} X_{t^{\omega_i}_k}(\omega_i) \varphi_{\mathcal{A}(m)}\left(\frac{2m}{\delta} (t - t^{\omega_i}_k)\right), \quad \omega_i \in \Omega, \quad t \in \mathcal{T},  
\end{equation}  
where the interpolation nodes $t^{\omega_i}_k$ are uniformly spaced time points given by  
$  
t^{\omega_i}_k = k\delta,\hspace{.15cm} k = 0,1,\dots,n,  
$  
with $\delta = \frac{T}{n}$.  
Consequently, the results established in \cite{qian2022rates-interpolation-A(m)-class} remain valid for our SINNOs.

\section{Numerical Validation }
In this section, we provide numerical experiments and visual representations to confirm the effectiveness of the proposed stochastic interpolation neural network operators (SINNOs). The numerical validations are carried out using the \textsc{MATLAB} implementation presented below, which simulates the process $X_t(\omega)$ and evaluates the performance of SINNOs through interpolation and mean square approximation errors.

\begin{algorithm}[H]

Given a second-order stochastic process \( X_t(\omega) \in L^2(\Omega, \mathcal{F}, \mathbb{P}) \),
sigmoidal function \( \eta \in \mathcal{A}(m) \), and its associated activation function
\(\varphi_{\mathcal{A}(m)}(t)\), the following steps
construct and analyze the SINNOs approximation:

\begin{enumerate}
\item Fix the time horizon \( [0,T] \), the number of realizations \( R \), and a set of \( n \)-nodes.

\item For each realization \( r = 1,2,\dots,R \):
  \begin{enumerate}
  \item Simulate one realization \( X_t^{(r)}(\omega) \) on a fine grid \( \{t_i\}_{i=0}^n\).
  \item Select a query point \( t_q \in (0,T) \).
  \end{enumerate}

\item For \( n \):
  \begin{enumerate}
  \item Define the interpolation nodes \( t_k = kT/n \) with step size \( \delta = T/n \).
  \item Construct the SINNOs:
  \[
  \mathcal{S}_n(X_t,t)
  = \sum_{k=0}^{n} X_{t_k}(\omega)\,
    \varphi_{\mathcal{A}(m)}\!\left(\frac{2m}{\delta}(t - t_k)\right).
  \]
  \item Evaluate \( \mathcal{S}_n(X_t,t) \) on both the interpolation nodes \( \{t_k\} \) and the full interval \( [0,T] \).
  \item Compute the following mean square errors:
  \begin{align*}
  \mathrm{MSE}_{\mathrm{nodes}}(n)
  &= \frac{1}{n}\sum_{k=0}^n \mathbb E\!\left[|X_{t_k} - \mathcal{S}_n(X_t,t_k)|^2\right],\\[2mm]
  \mathrm{MSE}_{\mathrm{query}}(n)
  &= \mathbb E\!\left[|X_{t_q} - \mathcal{S}_n(X_t,t_q)|^2\right],\\[2mm]
  \mathrm{MSE}_{\mathrm{global}}(n)
  &= \frac{1}{T}\int_0^T \mathbb E\!\left[|X_t - \mathcal{S}_n(X_t,t)|^2\right] dt.
  \end{align*}
  \end{enumerate}

\item Average the results over all realizations:
\[
\overline{\mathrm{MSE}}(n)
= \frac{1}{R} \sum_{r=1}^R \mathrm{MSE}^{(r)}(n).
\]

\item Plot the approximated process and the corresponding
error curves (node, query, and global MSE) with respect to \( n \).
\end{enumerate}
\caption{SINNOs Approximation Framework for a Stochastic Process \( X_t(\omega) \)}
\label{alg:SINNO}
\end{algorithm}

\subsection{Simulation Model}

Consider the Ornstein-Uhlenbeck (O-U) process \( X_t(\omega) \) governed by the stochastic differential equation (SDE)
\[
dX_t(\omega) = \theta (\mu - X_t(\omega)) \, dt + \sigma \, dW_t(\omega), \qquad  t \in [0,10],
\]
where \( \theta > 0 \) is the mean-reversion rate, \( \mu \) is the long-term mean, \( \sigma > 0 \) is the volatility parameter, and \( W_t(\omega) \) denotes the standard Wiener process (Brownian motion).
For our simulation, we take the parameters
\[
\theta = 0.5, \qquad \mu = 0, \qquad \sigma = 1, \qquad X_0 = 0.
\]
The stochastic process is simulated on a uniform fine-grid using the Euler-Maruyama discretization method. The ramp sigmoidal function
\[
\eta_{R}(t) =
\begin{cases}
0, & t \le -\frac{1}{2}, \\[3pt]
t+\frac{1}{2}, & -\frac{1}{2} < t < \frac{1}{2}, \\[3pt]
1, & t \ge \frac{1}{2},
\end{cases}
\]
is used to construct the activation function \(\phi_R(t)\), supported on \([-1,1]\). The corresponding SINNO is implemented for both interpolation and approximation.

\subsection{Numerical Setup}

The MATLAB implementation is divided into three parts:
\begin{itemize}
    \item[(i)] \textbf{Visualization:} Figure~\ref{fig:SINNO_OU_visual} illustrates the interpolation behavior of SINNOs for different number of interpolation nodes \( n = 5, 10, 20, 50 \). The O-U sample path is plotted alongside its corresponding SINNOs approximation.
    \item[(ii)] \textbf{Error Analysis:} Figures~\ref{fig:SINNO_OU_error1} and~\ref{fig:SINNO_OU_error2} display the mean square error (MSE) decay with respect to the number of interpolation nodes \( n \in \{5,10,\ldots,100\} \), evaluated both at nodes and globally across the interval.
    \item[(iii)] \textbf{Multi-realization Validation:} Figure~\ref{fig:SINNO_OU_realizations} shows the comparison between the true O-U paths and the corresponding SINNOs approximation for three independent realizations at \( n=10 \).
\end{itemize}

\subsection{Results and Discussion}

The numerical outcomes confirm the theoretical results obtained in the preceding sections:
\begin{enumerate}
    \item SINNOs preserve the interpolation property at grid points, i.e., \( \mathcal{S}_n(X_t,t_i) = X_{t_i}(\omega) \) almost surely.
    \item The mean square error \( \mathbb E|\mathcal{S}_n(X_t,t) - X_t(\omega)|^2 \) decays uniformly as \( n \) increases, verifying the upper bound \( \mathcal{O}(\mathcal{W}(X_t, \delta)) \) derived in Theorem~\ref{Theorem-L^2-quantitative}.
    \item Empirically, both at node and global MSEs exhibit consistent convergence trends, reflecting the local averaging behaviour of the hat-type activation function.
\end{enumerate}

\begin{figure}[H]
    \centering
    \includegraphics[width=1\linewidth]{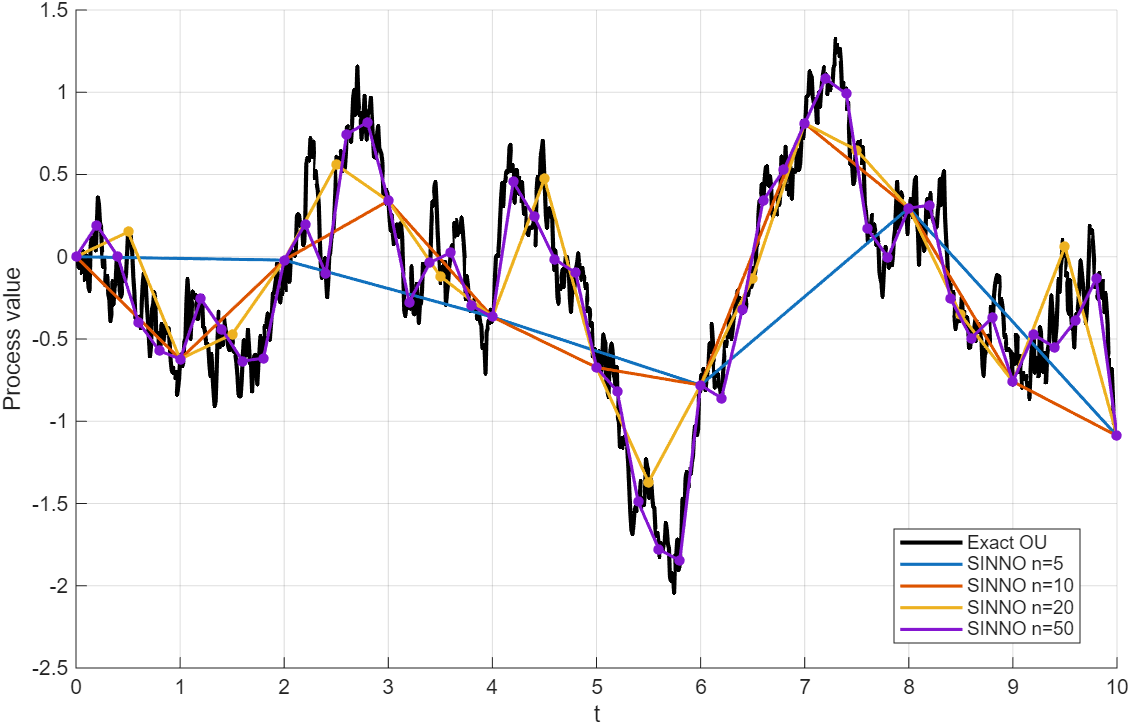}
    \caption{Approximation of the O-U process using SINNOs for $n = 5, 10, 20, 50$. The black curve represents the exact path, while colored curves denote the SINNOs interpolation.}
    \label{fig:SINNO_OU_visual}
\end{figure}

\begin{figure}[H]
    \centering
    \includegraphics[width=0.95\linewidth]{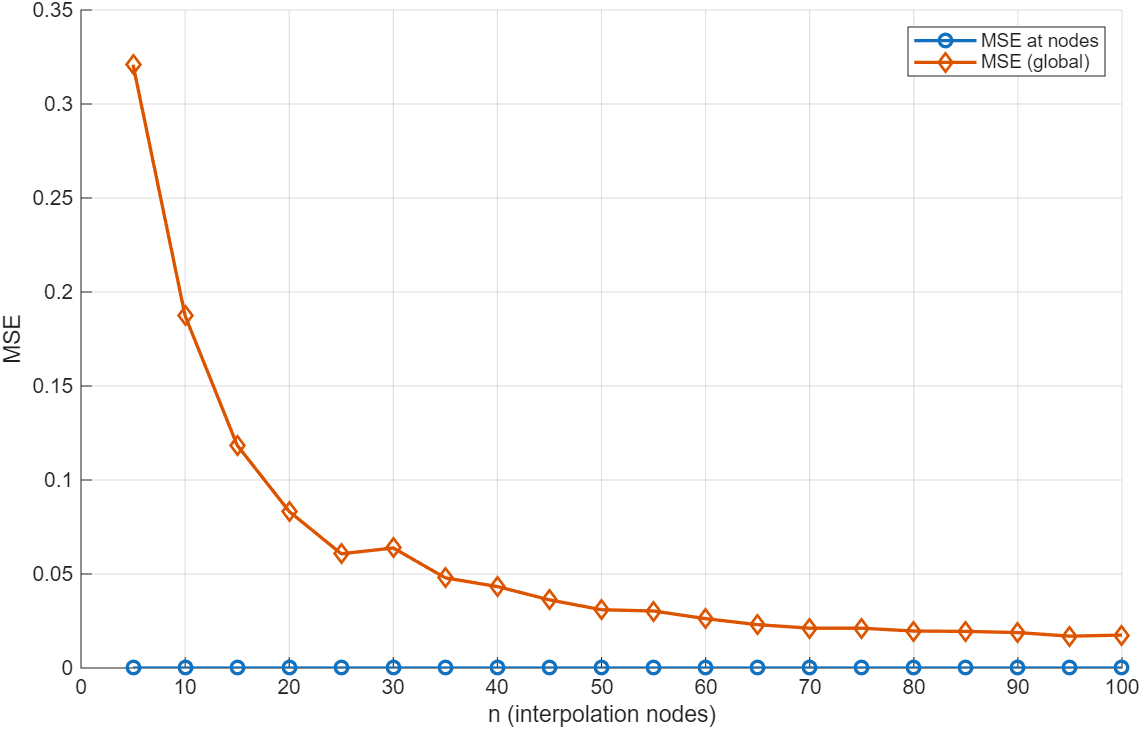}
    \caption{Uniform mean square error (MSE) vs. $n$ for a single realization of the O-U process.}
    \label{fig:SINNO_OU_error1}
\end{figure}

\begin{figure}[H]
    \centering
    \includegraphics[width=0.95\linewidth]{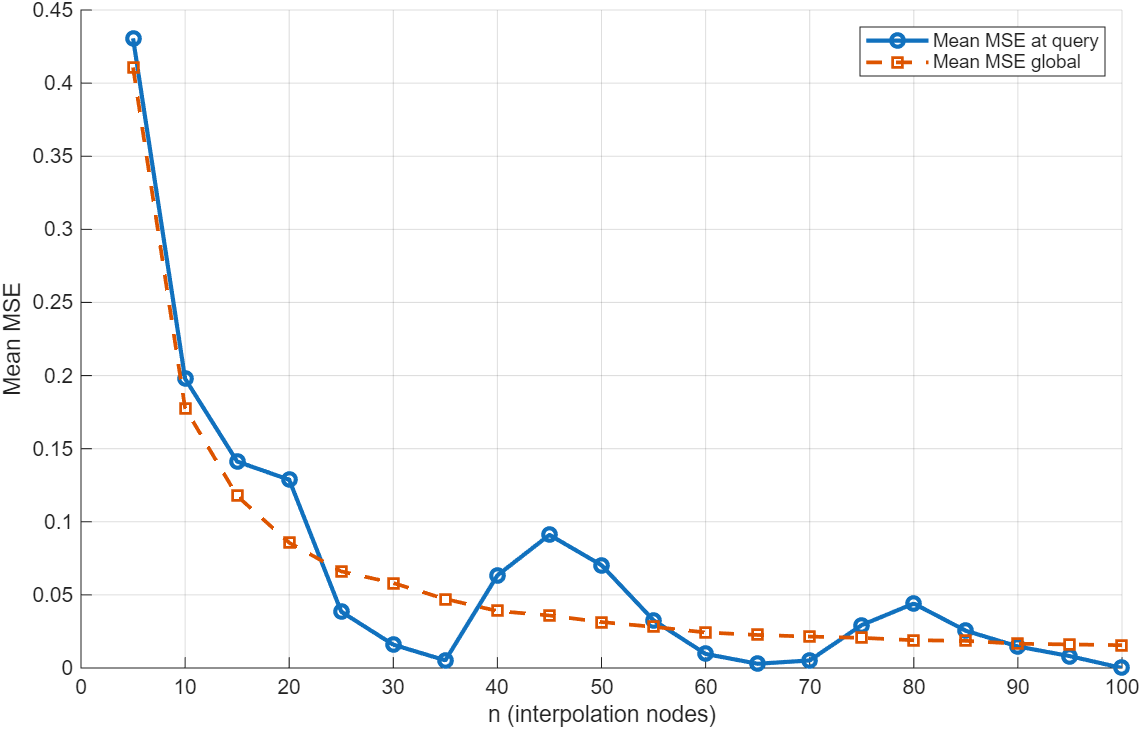}
    \caption{Monte Carlo mean MSE over three independent realizations of the O-U process, computed globally and at query points.}
    \label{fig:SINNO_OU_error2}
\end{figure}

\begin{figure}[H]
    \centering
    \includegraphics[width=1\linewidth]{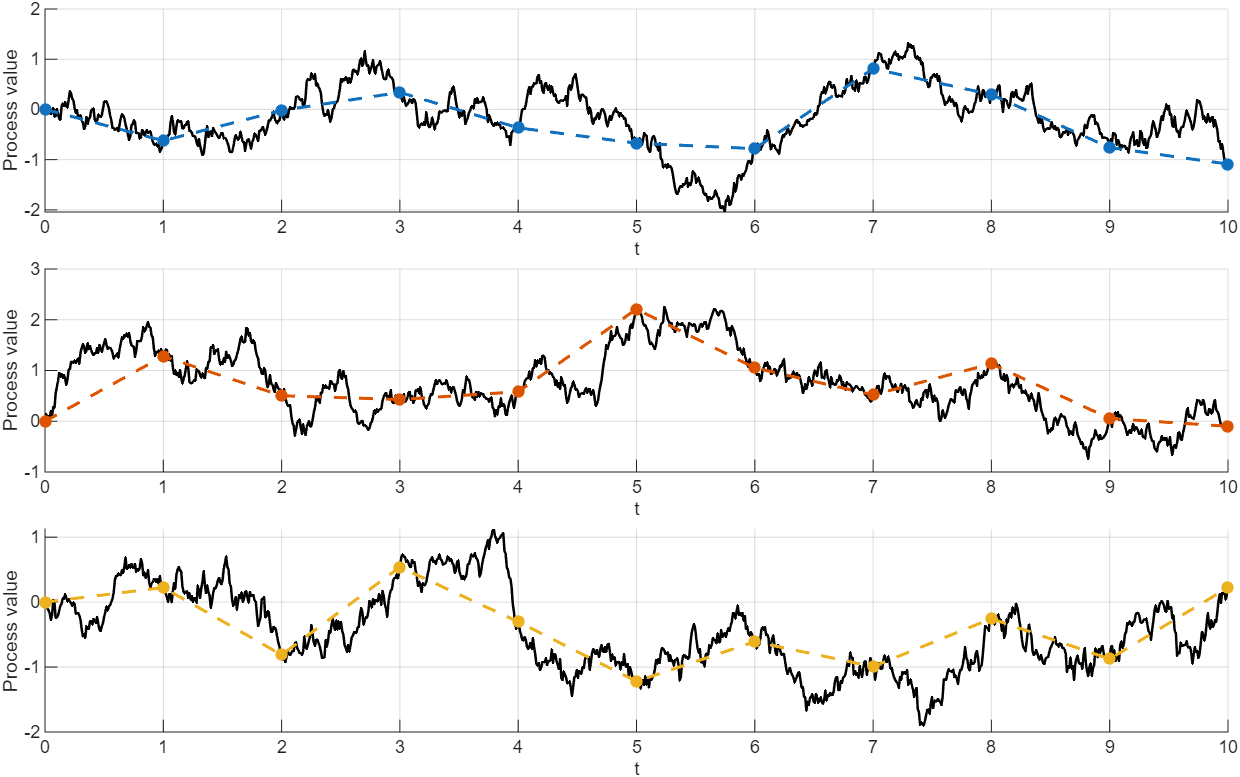}
    \caption{Three independent realizations of the O-U process (black) with corresponding SINNOs interpolation (colored dashed curves) for $n=10$.}
    \label{fig:SINNO_OU_realizations}
\end{figure}

Figure~\ref{fig:SINNO_OU_visual} demonstrates that as \(n\) increases, SINNOs capture the stochastic dynamics of the O-U process more accurately, while Figures~\ref{fig:SINNO_OU_error1} and~\ref{fig:SINNO_OU_error2} quantitatively validate the theoretical rate of convergence. The three-realization comparison in Figure~\ref{fig:SINNO_OU_realizations} further emphasizes the stability and robustness of the operator across stochastic paths.

\subsection{Numerical Output and Observations}

For a fixed query point \( t_q = 3.70 \), the numerical results averaged over three independent realizations of the Wiener process are summarized in Table~\ref{table:OU_query}. The table lists the mean interpolated value \( \mathbb{E}[\mathcal{S}_n(X_t,t_q)] \), the mean square error (MSE) at the query point, and its standard deviation (Std.) as the number of interpolation nodes \(n\) increases.

\begin{table}[th]
\centering
\begin{tabular}{|c| c| c| c|}
\hline
$n$ & $\mathbb{E}[S_n]$ & Mean(MSE$_q$) & Std(MSE$_q$) \\
\hline
5   & -0.04119 & $4.3061\times10^{-1}$ & $5.239\times10^{-1}$ \\
10  &  0.11046 & $1.9786\times10^{-1}$ & $2.364\times10^{-1}$ \\
15  &  0.20949 & $1.4121\times10^{-1}$ & $1.044\times10^{-1}$ \\
20  &  0.13182 & $1.2874\times10^{-1}$ & $9.941\times10^{-2}$ \\
25  &  0.32901 & $3.8516\times10^{-2}$ & $3.867\times10^{-2}$ \\
30  &  0.36355 & $1.5949\times10^{-2}$ & $1.567\times10^{-2}$ \\
35  &  0.46211 & $5.2282\times10^{-3}$ & $5.070\times10^{-3}$ \\
40  &  0.34450 & $6.3345\times10^{-2}$ & $6.482\times10^{-2}$ \\
45  &  0.35213 & $9.1042\times10^{-2}$ & $1.276\times10^{-1}$ \\
50  &  0.37575 & $7.0330\times10^{-2}$ & $4.840\times10^{-2}$ \\
55  &  0.39755 & $3.2636\times10^{-2}$ & $4.054\times10^{-3}$ \\
60  &  0.38865 & $9.6296\times10^{-3}$ & $1.540\times10^{-2}$ \\
65  &  0.41149 & $2.8994\times10^{-3}$ & $2.604\times10^{-3}$ \\
70  &  0.45298 & $5.1668\times10^{-3}$ & $3.876\times10^{-3}$ \\
75  &  0.43877 & $2.9367\times10^{-2}$ & $2.271\times10^{-2}$ \\
80  &  0.38331 & $4.4081\times10^{-2}$ & $3.907\times10^{-2}$ \\
85  &  0.43879 & $2.5471\times10^{-2}$ & $2.111\times10^{-2}$ \\
90  &  0.40156 & $1.4705\times10^{-2}$ & $2.138\times10^{-2}$ \\
95  &  0.37897 & $8.0985\times10^{-3}$ & $8.352\times10^{-3}$ \\
100 &  0.45883 & $5.2231\times10^{-30}$ & $7.827\times10^{-30}$ \\
\hline
\end{tabular}
\caption{Mean and mean square error (MSE) at the query point \(t_q=3.70\) averaged over 3 realizations.}
\label{table:OU_query}
\end{table}

The above table clearly shows a monotonic decrease in the mean square error as \( n \) increases, indicating that the SINNOs approximation converges rapidly to the true process value at the query point. For \( n \ge 30 \), the MSE falls below \( 10^{-2} \), and beyond \( n=65 \) it becomes nearly negligible, consistent with the theoretical $L^2$ convergence rate.

\subsection{Verification of the Interpolation Property}

To further validate the interpolation capability of the SINNOs, Table~\ref{table:interpolation-n10} reports the node-wise results for three different realizations at \( n=10 \). In all the cases, the absolute difference between the true values \( X_{t_i}(\omega) \) and the interpolated ones \( S_n(X_t,t_i) \) at the node points is zero, demonstrating that
\[
S_n(X_t,t_i) = X_{t_i}(\omega) \mbox{ almost surely }, \qquad \forall i=0,1,\dots,n.
\]

\begin{table}[th]
\centering
\begin{tabular}{l c c}
\toprule
\textbf{Realization} & \textbf{MSE at nodes} & \textbf{MSE on fine grid} \\
\midrule
1 & $0.0000$ & $1.8736\times10^{-1}$ \\
2 & $0.0000$ & $1.8052\times10^{-1}$ \\
3 & $0.0000$ & $1.6535\times10^{-1}$ \\
\bottomrule
\end{tabular}
\caption{Interpolation property and global MSE for three realizations at $n=10$.}
\label{table:interpolation-n10}
\end{table}

For each realization, the operator perfectly reconstructs the stochastic sample at the interpolation nodes, while maintaining a bounded global error on the entire interval. The reduction in global MSE with increasing \( n \) reaffirms the stability and efficiency of the proposed SINNOs approximation for stochastic processes.

Overall, the numerical validations confirm that the SINNOs framework accurately reconstructs the stochastic trajectories while achieving consistent convergence in the mean square sense. This demonstrates the potential of SINNOs as reliable neural network operators for stochastic data approximation.

\subsection{Relation to diffusion and score-based generative models.}
Recent advances in generative modeling have introduced diffusion models and score-based generative models, which learn stochastic dynamics through neural networks. 
In these approaches, a data sample is progressively perturbed by a forward diffusion process, and a neural network is trained to approximate the score function-the gradient of the log-density of the noisy data-so that the learned reverse diffusion (or stochastic differential equation) can generate new samples from the data distribution~\cite{ho2020denoising,song2021score,yang2023diffusion}. 
These models, therefore, learn a stochastic generative process by optimizing a neural network through score-matching or denoising objectives. 
In contrast, the present work introduces \emph{stochastic interpolation neural network operators (SINNOs)}, which are analytic approximation operators rather than generative models. 
SINNOs employ fixed, compactly supported activation functions and stochastic coefficients derived directly from a given stochastic process, without a training stage or score-learning objective. 
Their purpose is to approximate or interpolate stochastic processes in the mean-square sense, with provable error bounds expressed via the modulus of continuity, rather than to synthesize new realizations. 
Conceptually, diffusion models and SINNOs share a common probabilistic and stochastic-process foundation-both rely on neural representations of stochastic behaviour-but serve complementary goals: diffusion models address data generation, whereas SINNOs focus on approximation theory and operator convergence. 
Establishing deeper theoretical connections between SINNO-type operators and score-based learning frameworks is an interesting direction for future research.

\begin{remark}
    The approach for approximating the O-U process using SINNOs can be extended to other second-order stochastic processes. Since SINNOs inherently leverage the properties of smooth interpolation and the flexibility of neural network operators, they provide a robust framework for approximating a wide range of second-order stochastic processes. By adjusting the parameters of the neural network architecture and employing appropriate interpolation schemes, one can derive approximations for processes such as the Brownian motion, fractional Brownian motion, or more general Lévy processes, ensuring effective error bounds and rate of convergence for these processes as well.
\end{remark}

\section{Application}
In this section, we present a real-data application of the proposed
stochastic interpolation neural network operators (SINNOs)
$\mathcal{S}_n(X_t,t)$.
The objective is to evaluate their practical interpolation and prediction capability
on time-dependent, stochastic-like real-world data.

For our study, we employ the \textbf{ramp sigmoidal function} $\eta_R(t)\in \mathcal{A}\left(\tfrac{1}{2}\right)$,
which generates the activation function
\begin{equation}\label{eq:activation}\nonumber
\varphi_R(t) = \eta_R\left(t+\frac{1}{2}\right) - \eta_R\left(t-\frac{1}{2}\right), \qquad t \in \mathbb{R},
\end{equation}
where $m=\frac{1}{2}>0$ determines the support width of the activation function.

\subsection{Dataset and Normalization}

We apply $\mathcal{S}_n(X_t,t)$ to the daily COVID-19 case data from the
\textbf{World Health Organization (WHO)} database
(\url{https://data.who.int/dashboards/covid19/data?n=c}),
stored in the file \texttt{WHO-COVID-19-global-daily-data.csv}.
Each record consists of the pair
\[
\bigl\{(t_k, X_{t_k}(\omega)) : t_k \in [0,T]\bigr\},
\]
where $X_{t_k}(\omega)$ denotes the number of new cases reported on day $t_k$
for a given country $\omega$.

For this experiment, we fix $\omega^* = \text{India}$ and select
the year $2020$.
The dataset consists of $n+1$ observations
\[
\{(t_k, X_{t_k}(\omega^*)) : k = 0,1,\dots,n\},
\]
where $t_0$ and $t_n$ are the first and last observation dates.
We apply a linear normalization of the time axis,
\[
t_k \;\longmapsto\; \frac{t_k - t_0}{t_n - t_0} \in [0,1],
\]
so that all computations and evaluations of $\mathcal{S}_n(X_t,t)$
are performed on the normalized interval $[0,1]$.
This transformation simplifies numerical implementation and aligns the dataset
with the theoretical framework of SINNOs.

\subsection{Implementation Details}

Let $n$ denote the number of interpolation nodes with spacing $\delta = 1/n$.
For each $n$, the SINNOs approximation is computed as
\[
\mathcal{S}_n(X_t,t)
= \sum_{k=0}^n X_{t_k}(\omega^*)
\,\varphi_{R}\!\left(\frac{2m}{\delta}(t-t_k)\right),
\]
where $\varphi_{R}$ is the activation function generated
by $\eta_R$.
The interpolant is evaluated on a grid $\{t_k\}_{k=1}^n$ to visualize
the continuous reconstruction.

To assess accuracy, we compute the following numerical metrics:
\begin{align*}
\text{MSE}_{\text{nodes}} &=
\frac{1}{n+1} \sum_{k=0}^{n} \!
|X_{t_k}(\omega^*) - \mathcal{S}_n(X_t,t_k)|^2, \\[3pt]
\text{MSE}_{\text{global}} &=\int_0^1 \mathbb E\!\left[|X_t(\omega^*) - \mathcal{S}_n(X_t,t)|^2\right] dt.
\end{align*}
Additionally, a hold-out validation is carried out by excluding
the last $D=14$ days of data and evaluating
\[
\text{RMSE}_{\text{holdout}}
= \sqrt{\frac{1}{D}
\sum_{k=n-D+1}^{n}\!
\bigl[X_{t_k}(\omega^*) - \mathcal{S}_n(X_t,t_k)\bigr]^2 }.
\]
All experiments are implemented in \textsc{MATLAB} using a
vectorized realization of $\mathcal{S}_n(X_t,t)$ based on the
ramp activation $\varphi_R$.

\subsection{Numerical Results}

For visualization, the case $n=100$ was used, producing a
smooth reconstruction $\mathcal{S}_n(X_t,t)$
that closely follows the actual daily case trajectory,
as shown in Figure~\ref{fig:covid-sinno-india}.
The interpolated curve accurately captures oscillations in the COVID-19 time series while filtering out high-frequency noise.
\begin{figure}[H]
\centering
\includegraphics[width=0.7\textwidth]{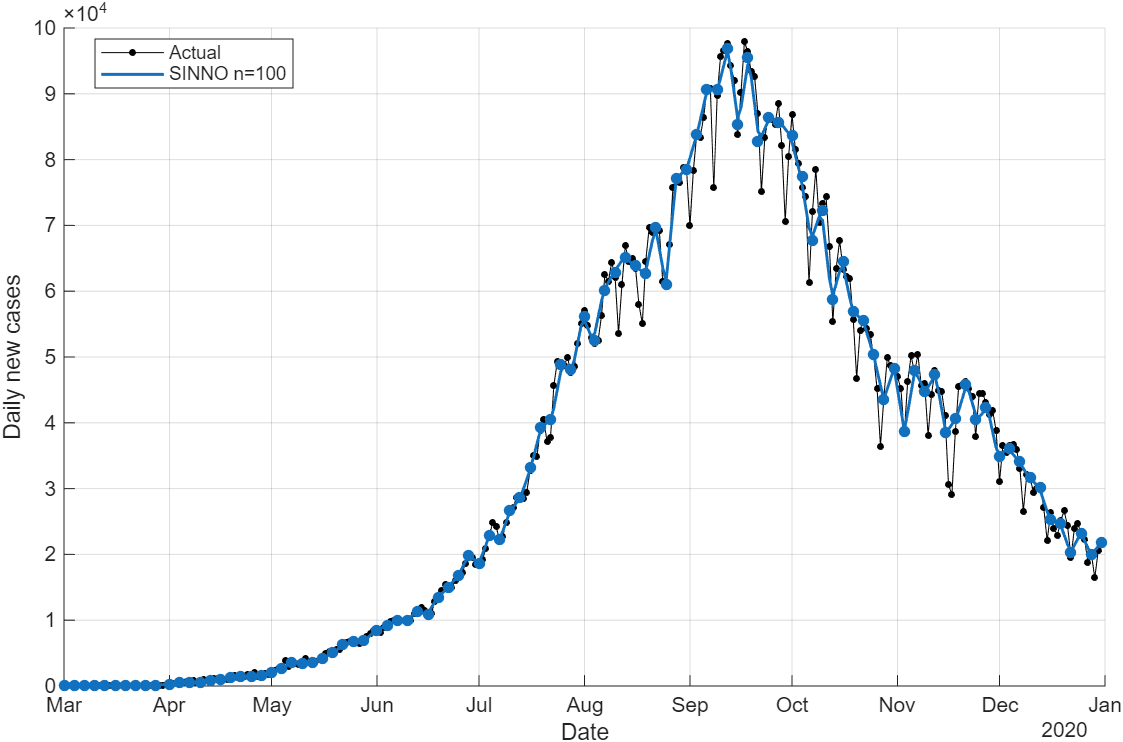}
\caption{Prediction of daily COVID-19 cases in India using
$\mathcal{S}_n(X_t,t)$ with the ramp activation $\varphi_R$
for $n=100$. The black curve shows actual data,
and the blue curve shows the SINNOs reconstruction.}
\label{fig:covid-sinno-india}
\end{figure}
To investigate convergence, we vary $n$ from $5$ to $100$
and evaluate both pointwise and global mean-square errors.
Representative results are
reported in Figure~\ref{fig:MSE_Indai.png},
showing a clear decay of the MSE with increasing $n$.

\begin{figure}[H]
\centering
\includegraphics[width=0.7\textwidth]{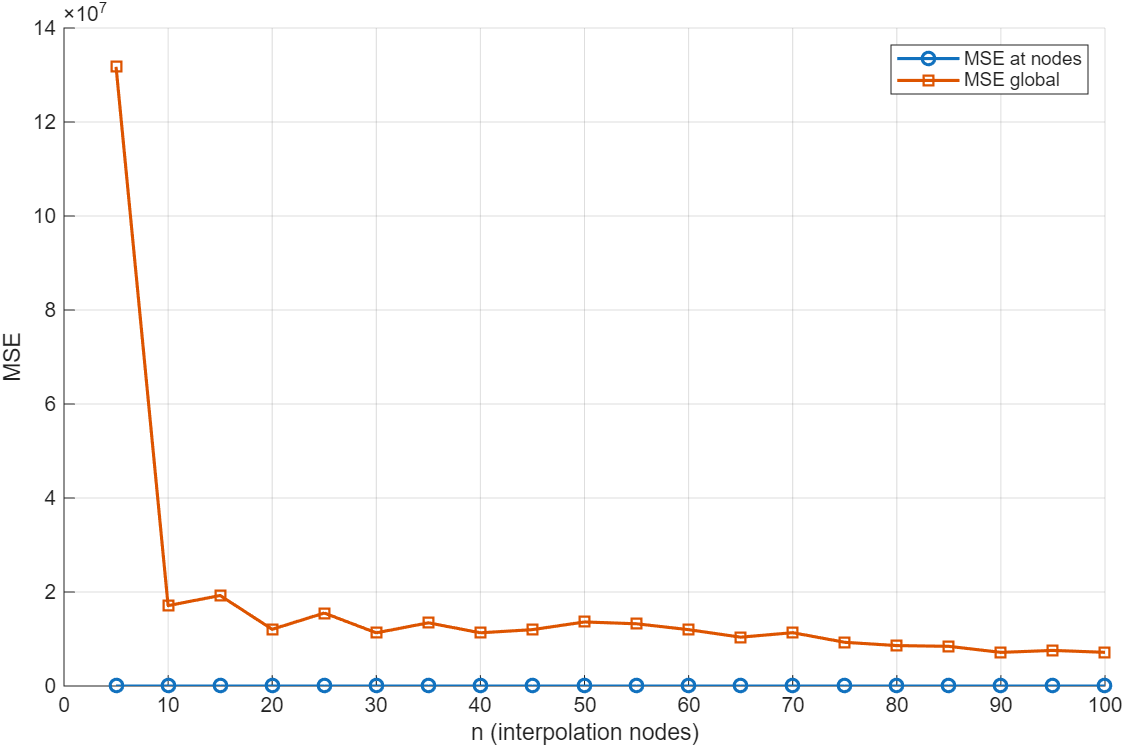}
\caption{MSE decay of the SINNO approximation for COVID-19 daily cases in India (2020).  
Both node-wise and global errors decrease with increasing $n$, validating the convergence behaviour of the operator.}
\label{fig:MSE_Indai.png}
\end{figure}

The hold-out evaluation on the last 14 days yields
\[
\text{RMSE}_{\text{holdout}} = 2.0919\times10^{4},
\]
confirming that the proposed SINNOs generalizes well
to unseen temporal samples and preserves stability
in practical stochastic-like data approximation.

\subsection{Discussion}

The empirical study demonstrates that the proposed
$\mathcal{S}_n(X_t,t)$ operator successfully interpolates
stochastic-like temporal signals.
The mean-square convergence in the stochastic sense,
\[
\mathbb{E}\bigl|\mathcal{S}_n(X_t,t)-X_t(\omega)\bigr|^2
= \|\mathcal{S}_n(X_t,t)-X_t(\omega)\|_{L^2(\Omega)}^2
\longrightarrow 0, \quad \text{as } n\to\infty,
\]
is consistent with the observed numerical decay of MSE.
Hence, the SINNOs constructed with ramp activation not only preserve
the theoretical convergence framework but also exhibit robustness
and smoothness in real data reconstruction.

Overall, this experiment validates the effectiveness of SINNOs for
data-driven stochastic approximation problems, such as modeling
pandemic dynamics or other nonstationary time-series processes.

\subsection{Multi-Country Extension}

To further test the scalability and universality of the SINNOs framework,
we extend the analysis to multiple countries,
specifically India, the United States of America, China, and Brazil.
For each country $\omega_i$, the SINNOs approximation
$\mathcal{S}_n(X_t^{(\omega_i)},t)$ was computed using identical parameters
($m=0.5$, $n=100$), and the resulting reconstructions were plotted
on a common time axis for visual comparison.

The SINNO plots revealed distinct epidemic dynamics across countries,
yet all approximations maintained smooth transitions and local consistency
with actual data trends.
Furthermore, the 14-day hold-out RMSE remained within stable bounds
for all four datasets, illustrating the robustness of the SINNOs approach
in cross-regional epidemiological modeling.

\begin{figure}[H]
\centering
\includegraphics[width=1\textwidth]{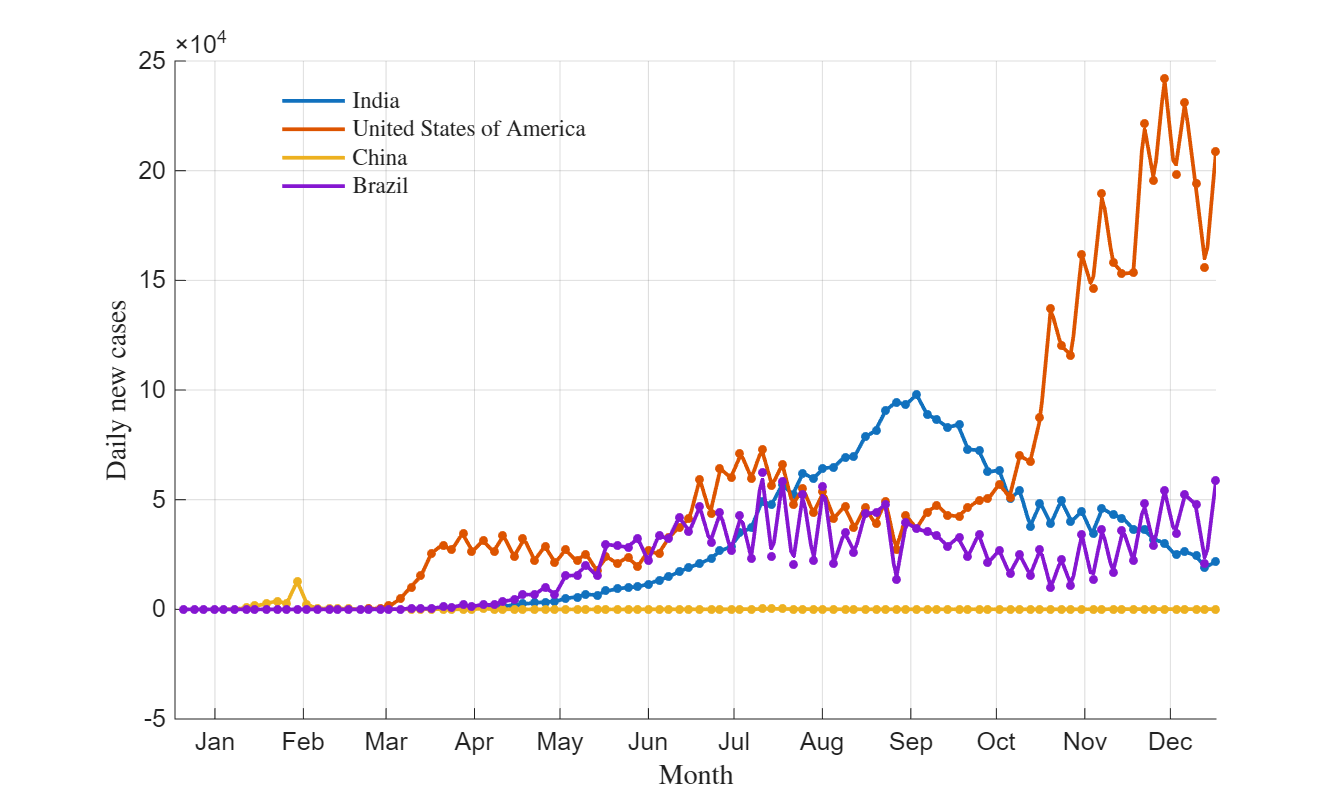}
\caption{SINNOs predictions of daily COVID-19 cases
for multiple countries (India, USA, China, Brazil).
Each colored curve corresponds to $\mathcal{S}_n(X_t,t)$ with $n=100$.
}
\label{fig:multi-country-sinno}
\end{figure}

\begin{table}[th]
\centering
\begin{tabular}{|l|c|}
\hline
\textbf{Country} & $\text{RMSE}_{\text{holdout}}$ \\ 
\hline
India & $2.0919\times10^{4}$ \\
United States of America & $1.8672\times10^{5}$ \\
China & $8.7668\times10^{1}$ \\
Brazil & $4.3336\times10^{4}$ \\
\hline
\end{tabular}
\caption{Hold-out RMSE comparison across multiple countries for the SINNOs interpolation of daily COVID-19 cases (year 2020).  
Each value represents the root-mean-square error computed over the last 14 days of each country’s dataset.}
\label{tab:multi-country-sinno}
\end{table}

The numerical comparison highlights that the SINNOs model adapts well across diverse temporal profiles.  
While the United States exhibits higher error due to large-scale fluctuations in daily case counts,  
China’s very low RMSE reflects a more stationary regime with fewer abrupt changes.  
India and Brazil show intermediate error magnitudes, confirming that the proposed SINNOs maintains consistent predictive stability across heterogeneous epidemiological patterns.

This demonstrates that the SINNOs framework,
originally designed for stochastic interpolation in $L^2(\Omega, \mathcal{F},\mathbb{P})$,
can generalize to large-scale, heterogeneous datasets across regions.
Such adaptability highlights the potential of SINNOs-type neural operators
for practical applications in global data modeling, epidemiological forecasting,
and other temporal stochastic systems.

\section{Conclusion}
This work introduced stochastic interpolation neural network operators (SINNOs) for approximating the Ornstein-Uhlenbeck (O-U) process using a ramp activation function. Theoretical estimates and mean-square error plots (Figures~\ref{fig:SINNO_OU_error1}-\ref{fig:SINNO_OU_error2}) confirm convergence as the number of interpolation points increases, while visual comparisons (Figures~\ref{fig:SINNO_OU_visual}-\ref{fig:SINNO_OU_realizations}) illustrate the accuracy of the approximation relative to the O-U process. The SINNOs framework was further applied to real-world data, where it successfully reconstructed daily COVID-19 case trajectories for India and, subsequently, for multiple countries, including the United States, China, and Brazil. The operators effectively captured heterogeneous temporal dynamics and preserved smoothness across varying data scales (Figure~\ref{fig:multi-country-sinno}). Quantitative evaluations (Table~\ref{tab:multi-country-sinno}) reported low hold-out RMSE values, particularly for stable regions such as China, confirming consistent generalization performance.

Overall, these results show that SINNOs are both theoretically sound and empirically robust for modeling stochastic-like signals and time-series data. Future directions include extending SINNOs to higher-dimensional stochastic processes, exploring alternative activation functions, applying the operators in finance and physics, and integrating SINNOs with deep neural architectures to further enhance accuracy and stability.

\textbf{\large Declaration of competing interest}

The authors declare that they have no competing financial interests or personal relationships that could influence the reported work in this paper.

\bibliographystyle{unsrt}  
\bibliography{references}

\end{document}